\documentclass{article}

\usepackage[nonatbib,preprint]{neurips_2021}

\usepackage[utf8]{inputenc} 
\usepackage[T1]{fontenc}    
\usepackage{hyperref}       
\usepackage{url}            
\usepackage{booktabs}       
\usepackage{amsfonts}       
\usepackage{nicefrac}       
\usepackage{microtype}      

\usepackage{amsthm}
\usepackage{amsmath}
\usepackage{xcolor}
\usepackage[capitalize]{cleveref}
\usepackage{graphicx}
\usepackage{algorithm}
\usepackage[noend]{algpseudocode}
\usepackage{subcaption}

\algrenewcommand\algorithmicindent{0.6em}

\title{Optimizing Graph Transformer Networks with Graph-based Techniques}

\author{
  Loc Hoang \thanks{This work was done by this author as part of work with AbbVie.} \thanks{Equal contribution}\\
  The University of Texas at Austin\\
  \texttt{loc@cs.utexas.edu} \\
  \And
  Udit Agarwal\footnotemark[2], Gurbinder Gill, Roshan Dathathri\\
  KatanaGraph\\
  \texttt{\{udit,gill,roshan\}@katanagraph.com}\\
  \And
  Abhik Seal, Brian Martin\\
  AbbVie Inc.\\
  \texttt{\{abhik.seal,brian.martin\}@abbvie.com}\\
  \And
  Keshav Pingali\\
  The University of Texas at Austin\\
  \texttt{pingali@cs.utexas.edu}\\
}

\newtheorem{theorem}{Theorem}

\begin{document}

\maketitle


\begin{abstract}
Graph transformer networks (GTN) are a variant of graph convolutional networks
(GCN) that are targeted to heterogeneous graphs in which nodes and edges have
associated type information that can be exploited to improve inference accuracy.
GTNs learn important metapaths in the graph, create weighted edges for these
metapaths, and use the resulting graph in a GCN. Currently, the only available
implementation of GTNs uses dense matrix multiplication to find metapaths.
Unfortunately, the space overhead of this approach can be large, so in practice
it is used only for small graphs. In addition, the
matrix-based implementation is not fine-grained enough to use random-walk based
methods to optimize metapath finding. In this paper, we present a
\emph{graph-based formulation} and implementation of the GTN metapath finding
problem. This graph-based formulation has two advantages over the matrix-based
approach.  First, it is more space efficient than the original GTN
implementation and more compute-efficient for metapath sizes of practical
interest. Second, it permits us to implement a \emph{sampling} method that
reduces the number of metapaths that must be enumerated, allowing the
implementation to be used for larger graphs and larger metapath sizes.
Experimental results show that our implementation is $6.5\times$ faster than the
original GTN implementation on average for a metapath length of 4, and our
sampling implementation is $155\times$ faster on average than this
implementation without compromising on the accuracy of the GTN. 
\end{abstract}


\section{Introduction}
\label{sec:intro}

Graph neural networks (GNNs) are used for machine learning on graphs to
perform tasks such as node classification and link prediction that
require learning features on nodes and/or edges. Many GNN architectures exist
such as the graph convolutional network (GCN)~\cite{GCN},
GraphSAGE~\cite{GraphSAGE}, and GIN~\cite{GIN}, among others~\cite{FastGCN,GAT,GraphSAINT}.  One
limitation of most GNN architectures is that they do not make use of
the additional information present in a heterogeneous graph,
such as the type information on vertices and/or edges.
Being able to operate on such heterogeneous graphs is important: for instance,
heterogeneous knowledge graphs have become the preferred choice for representing
data in the biomedical domain because they can assimilate data from many sources and
model the edge semantics needed to find important relations like drug-target
pairs, drug-side-effect pairs, drug-disease pairs, disease-pathway pairs, etc.
Another example is finding key opinion leaders who can champion a topic,
which is often done by using heterogeneous graphs compiled from literature,
publications, patents, scientific research areas, and citations.

The graph transformer network (GTN)~\cite{gtn} can leverage type information in
heterogeneous graphs by learning important \emph{metapaths} (typed paths) in the graph
and encoding this information into a \emph{metapath graph} which can then be
used by a regular GCN. This can improve node classification accuracy~\cite{gtn}.
In addition, they are able to \emph{find} important metapaths 
without the need for domain experts to provide the list of metapaths. Existing GTN implementations use dense matrix multiplication to compute metapaths. This is
not memory efficient as the size of the matrix grows quadratically with
the number of vertices in the graph. As a result, these implementations can be used only for
small graphs with fewer than 200K vertices. In addition, these matrix-based approaches
do not support fine-grained operations such as sampling paths through random
walks to reduce the number of considered paths.

This work makes the following contributions to the area of GTNs.

\begin{enumerate}
    \item We present a new algorithm for the graph transformer network
    that formulates the problem as a series of graph operations rather than as
    matrix operations.
    \item We present a random walk based approach that uses this graph-based formulation
     to sample important metapaths to further reduce memory usage and computation cost.
    \item We implement this algorithm and show that it outperforms the
    original implementation by $6.5\times$ on average.  We also show
    that random-walk sampling improves performance by $155\times$ over
    the original implementation without compromising accuracy of node classification.
    \item We show experimentally that the sampling approach can run and scale on large
    graphs with up to 1.5 billion edges.
\end{enumerate} 


\section{Background: Graph Transformer Networks}
\label{sec:background}

This section introduces graph neural networks and graph transformer networks,
and motivates the need for a graph formulation of GTNs.

%

\subsection{Graph Neural Networks}

Graph neural networks (GNNs)~\cite{originalgnn} extend the deep neural network
(DNNs) approach to graphs. In DNNs, an input (typically a tensor) is passed through a series of
\emph{layers} where it is transformed via tensor operations like matrix
multiplication using learnable parameters. The output at the end of the
layers is a \emph{transformed} version of the input used for some downstream
task.
DNNs are trained in a series of epochs that update the learnable parameters of
the layers to achieve better results such as improved accuracy for tasks like node
classification. In a GNN, the input to the layers are vertices and their features. To leverage
the fact that the input is a graph, the feature passed into a layer is not the
vertex feature by itself but an \emph{aggregation} of the features of the
vertex's neighbors. Many GNN architectures exist~\cite{GraphSAGE,GCN,GAT,GIN},
and they differ in how they define the \emph{aggregation} and the \emph{update} phases
(e.g., layer transformation including and pre-/post-processing). In this paper, we focus
on the graph convolutional network (GCN)~\cite{GCN}, a basic GNN which does a sum aggregation
of neighbor features normalized by source/destination vertex degrees, followed by a
vanilla update.


\subsection{Metapaths}

A length $l$ metapath in a heterogeneous graph is a path comprised of edge types $(t_1, t_2, ..., t_l)$ that represents a typed relation between the endpoints.  A \emph{metapath
edge} can be drawn between the endpoints of the metapath.  This explicit
representation of a metapath is useful because it (1) captures heterogeneity
with a non-typed edge and (2) explicitly encodes a multi-hop relationship beyond
a vertex's immediate neighborhood. Since most GNN architectures
do not leverage the heterogeneity of a graph during training, using metapath
edges to represent these heterogeneous relationships can increase the
information that the GNN can use during training and inference~\cite{magnn,gtn}.
Metapaths are also useful in other machine learning applications such as
improved node embeddings with metapath-guided random walks~\cite{hetespaceywalk,metapath2vec}
and metapath connections serving as predictors of effectiveness in a learning
model of drug effectiveness for diseases~\cite{drugmeta}.

Not all metapaths are meaningful, so typically an expert familiar with
the heterogeneous graph defines the important metapath relationships.
This can be problematic: it is not the case that an expert can be found for
every heterogeneous graph dataset, and even experts can introduce biases into
the metapaths. For datasets in a new domain, an expert may not even exist.
Therefore, it is useful to have a method to determine important metapaths
automatically and leverage them for learning. 

\subsection{Graph Transformer Networks}
\label{subsec:gtn}

GTN is a variant of graph convolutional networks (GCN)~\cite{GCN} that 
automatically learns important metapaths in a heterogeneous graph. Importance is defined
using a scoring function for paths. Usually, the score for a path
is computed as the product of the importance scores of its component edges, and
the importance score of an edge may depend on the position of that edge
in the metapath. Therefore, a length $l$ metapath $m$ comprised
of types $(t_1, t_2, ..., t_l)$ has an \emph{importance score} equal to
$s(1, t_1) * s(2, t_2) * \cdots * s(l, t_l)$, where $s$ is a scoring function
for an edge in a given position in the metapath~\cite{gtn}. 

\begin{figure*}[ht]
	\begin{center}
		\includegraphics[width=0.84\textwidth]{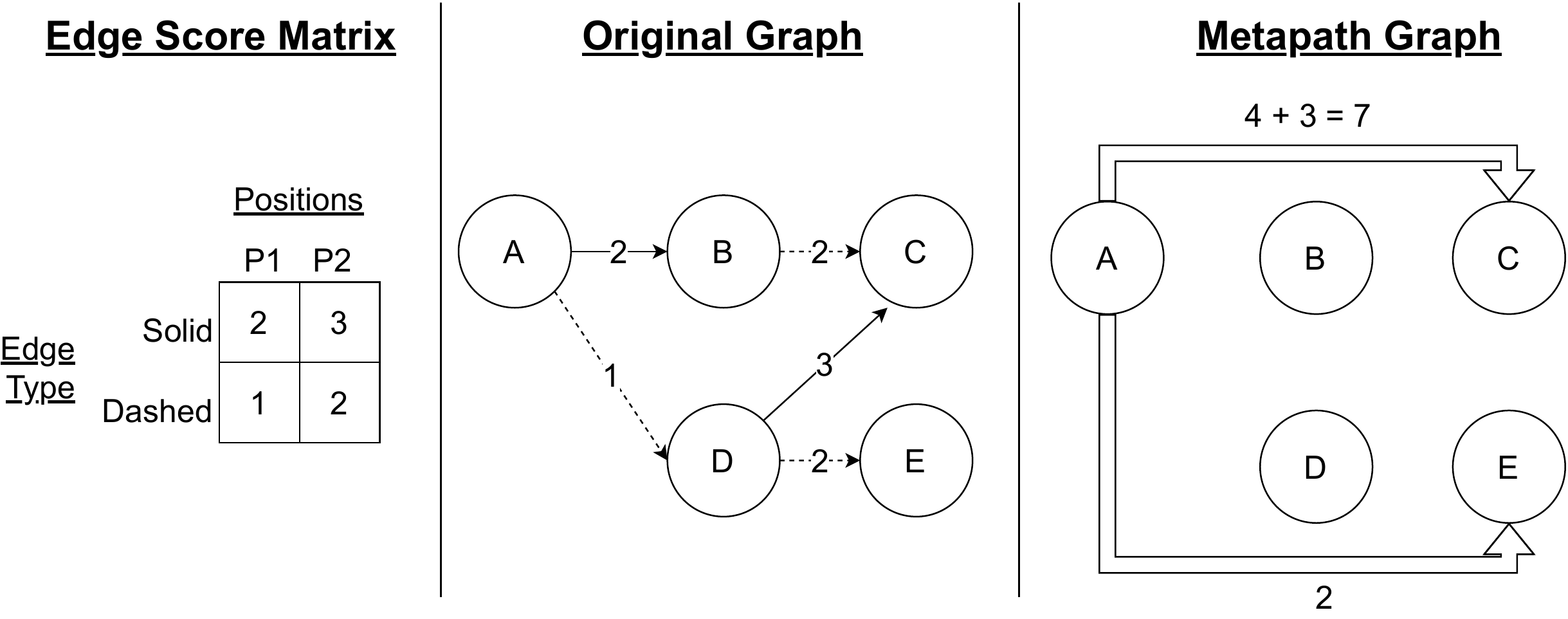}
		\caption{Metapath graph construction. The matrix denotes the
        importance of an edge type for a position in the metapath. The heterogeneous graph's
        edges are scored based on this matrix. The metapath graph has edges $(A, C)$ and $(A,E)$:
        the former is composed of $(A,B,C)$ and $(A,D,C)$ with scores $2 * 2  = 4$ and $1 * 3 = 3$, and the latter is composed of
        $(A, D, E)$ with score $1 * 2 = 2$.}
		\label{fig:metapath_graph}
	\end{center}
\end{figure*}
\begin{figure*}[ht]
	\begin{center}
		\includegraphics[width=0.64\textwidth]{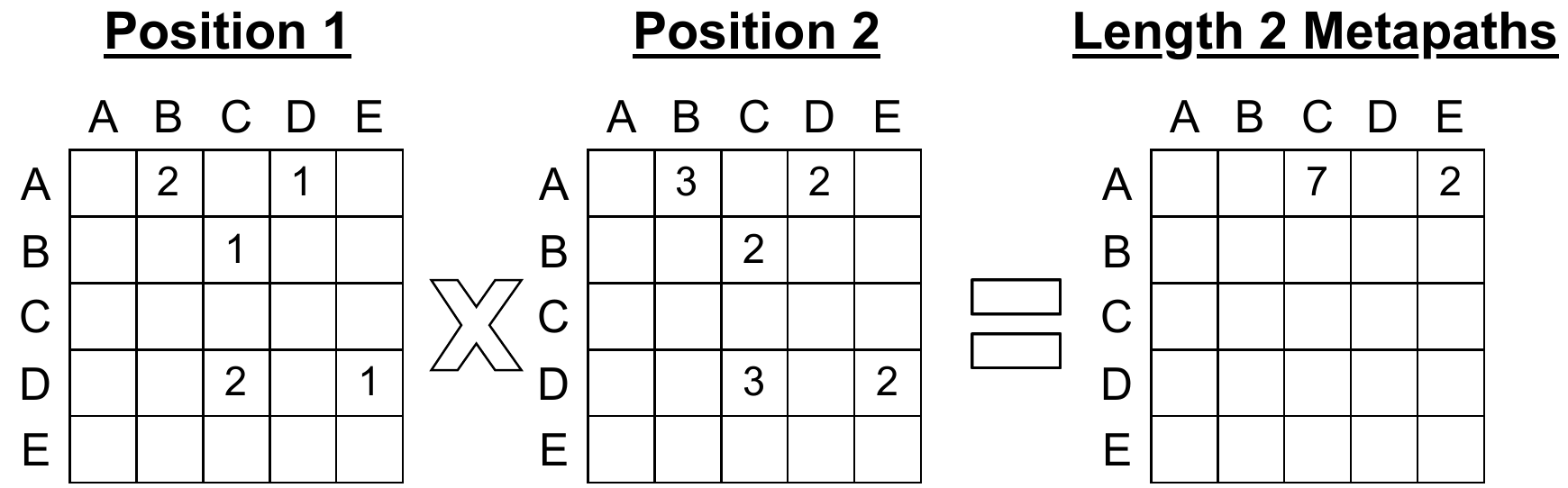}
		\caption{Example of matrix computation that occurs to find the
		metapaths in \cref{fig:metapath_graph}. The graph's adjacency matrix
		is duplicated, and the scores corresponding to each edge type for a
		particular on is filled accordingly. A matrix multiply then finds
		metapaths edges with the correct score: $(A, C) = 7$ and $(A, E) = 2$.}
		\label{fig:matrix_example}
	\end{center}
\end{figure*}

Given the input graph, the GTN generates a new \emph{metapath graph} by replacing the original graph's edges with weighted metapath edges. \cref{fig:metapath_graph} illustrates this process for a length 2 metapath.\footnote{A GTN implementation includes other details such as
addition of self-edges to the original graph to simulate paths of length less than
$l$ and softmax normalization of scores, as used in the original GTN
implementation~\cite{gtn}.  For simplicity of explanation and
brevity, we do not delve into these details, but the
implementations we evaluate in our experiments do account for them.}
The score matrix on the left is a graphical representation of an
edge scoring function $s$. The example has two edge types: solid and dashed.
The original graph is shown in the middle, and it has 3 paths: $(A, B, C)$,
$(A, D, C)$, and $(A, D, E)$. These correspond to 3 different metapaths: (solid, dashed),
(dashed, solid), and (dashed, dashed), respectively, and they get scores of
4, 3, and 2, respectively. The metapath edges are weighted based
on scores of the paths and their endpoints: $(A,B,C)$ and $(A,D,C)$ contribute
their scores to metapath edge $(A,C)$, and $(A,D,E)$ contributes to edge $(A, E)$.
The generated metapath graph will then be used as the input to the GCN in place
of the original graph. The metapath graph's edges are used during GNN aggregations, which
are \emph{weighted} based on the metapath edges' weight; therefore, if the
metapath weights accurately reflect the importance of the metapath, then
the GCN is able to correctly leverage the heterogeneity of the original
graph and its important relationships through the weights on metapath edges.

The scoring function must be able to score the edges accurately;
since the metapath graph construction is part of the inference pipeline, it is
possible to use the error of the GCN step for the task it is being trained for
to adjust the scoring function. This is the key idea behind GTN training: the error
from the GCN is back-propagated to the metapath graph generation step, and the
gradients can then be used to adjust the scores so that the next training epoch
will result in higher accuracy. There are two main benefits: (1) the GCN will be
augmented by heterogeneous metapath information that it did not have access to
before, which may improve GCN accuracy for heterogeneous graphs, and (2) the
scoring function is trained alongside the GCN, meaning that the scores can
identify the important edges of a metapath; in other words, \emph{important
metapaths are learned without any expert intervention}.

\subsection{Implementations of GTNs}

The only implementation of the GTN we know of is
the original one written in PyTorch~\cite{PyTorch,gtn}. The metapath
graph in this implementation is represented as a dense adjacency matrix
computed with a series of dense matrix multiplication operations. \cref{fig:matrix_example}
illustrates this for the example in \cref{fig:metapath_graph}: an edge score matrix
is created for every position in the metapath with the appropriate scores for each edge,
and the matrix multiplication operation is used to enumerate paths. For an $l$
length metapath, there are $l$ matrices and $l-1$ matrix multiplies to generate the
metapath graph.

In this approach to metapath finding, the computation cost
grows linearly with the length of the metapath, and the key computation is a
dense matrix multiplication, which has been heavily optimized by BLAS libraries.
However, the memory overhead of this approach is extremely high: dense matrices
use $O(n^2)$ memory where $n$ is the number of graph vertices.
Because of this, this implementation fails to run on
even average-sized graphs unless the machine has a very large memory. 

\section{Graph Formulation for Graph Transformer Networks}
\label{sec:gtn-graph}

This section presents our formulation of the graph transformer network as a
series of graph operations.

\subsection{Metapath Graph Generation via Path-finding}

Instead of finding metapaths via matrix multiplies, we find paths
via graph traversal and generate the metapath graph using the generated paths.
We use dynamic programming to make this process efficient.

\cref{alg:vanilla_gtn} shows the high level algorithm.  For a metapath graph
with metapaths of length $l$, the weight of a metapath edge $(a, b)$ is the
sum of the scores of all $l$-length metapaths starting at $a$ and ending at $b$.  Therefore, the
algorithm finds all length $l$ paths in the graph
(Line~\ref{algline:find_path}); the method of path finding on the graph is
left to the implementation, and different path finding methods trade-off
memory and compute overheads.
Each path is scored by obtaining
the individual scores of its edges and multiplying them together as discussed
in \cref{subsec:gtn}
(Line~\ref{algline:v_paths}).  This score is
added to a running sum for the metapath edge corresponding to that path
(Line~\ref{algline:v_metapath_add}).  After all paths are enumerated, scored,
and added, the metapath graph MG will be complete and have the correct edge
weights.

If $l$ is sufficiently large, path finding becomes very expensive as the
number of paths grows exponentially.  Therefore, we formulate
\cref{alg:opt_gtn}, based on a dynamic programming approach for constructing
large paths by concatenating several smaller paths~\cite{agarwal2018,cormen2009}
to reduce the computational requirement of
\cref{alg:vanilla_gtn} by cutting the length of the generated paths by half in
exchange for using more memory to store intermediate metapath graphs, called MG$_1$
and MG$_2$ in the pseudocode in \cref{alg:opt_gtn}.  Instead of enumerating length $l$ paths, paths of length $l/2$
are enumerated (Line~\ref{algline:o_enum}).
Then, two scores are generated
for each individual path:
the first one corresponds to the path starting from edge $(v,x_1)$ at position 1
(Line~\ref{algline:o_s1}) and the second one corresponds
to the one with all edge positions offset by $l/2$ (Line~\ref{algline:o_s2}).
These two scores are added into the metapath edges of two different intermediate metapath
graphs (Lines~\ref{algline:o_metaadd_start}-\ref{algline:o_metaadd_end}).
The final step (Lines~\ref{algline:o_final_start}-\ref{algline:o_final_end})
gets the metapath edge $(a, c)$ of the desired length $l$ metapath graph by
multiplying the scores of metapath edge $(a,b)$ and $(b, c)$ in the two
intermediate metapath graphs, respectively, for all $b$.  \cref{alg:opt_gtn}
computes the same metapath graph as \cref{alg:vanilla_gtn}: at a high level,
it is because the multiplication of the two metapath edge scores computes the
sum of the products of the same paths found by \cref{alg:vanilla_gtn} (a
formal proof can be found in the Appendix).

\begin{algorithm}[H]
\small
\caption{\small Vanilla Metapath Graph Generation}\label{alg:vanilla_gtn}
\textbf{Input:} Graph $G$; Edge Score Function $ES$; Edge Type Function $ET$\\
\textbf{Output:} Metapath Graph $MG = (V, E, W)$
\begin{algorithmic}[1]
  \For{all vertices $v$ in G}
    \State Enumerate length $l$ paths $P$ from $v$ \label{algline:find_path}
    \For{path $p=(v, x_1, ..., x_{l})$ in $P$}
      \State $score$ $=$ $ES(ET(v, x_1), 1) \cdot \prod\limits_{i=2}^{l} ES(ET(x_{i-1},x_{i}), i)$ \label{algline:v_paths}
      \State add edge $(v, x_l)$ to $MG.E$, if it already doesn't exist
	\State  $MG.W(v, x_{l}) += score$     \label{algline:v_metapath_add}
    \EndFor
  \EndFor
\end{algorithmic}
\end{algorithm}

\begin{algorithm}[H]
\small
\caption{\small Metapath Graph Generation using Random Walks}\label{alg:gtn_random}
\textbf{Input:} Graph $G$; Edge Score Function $ES$; Edge Type Function $ET$; Number of walks $num\_walks$\\
\textbf{Output:} Metapath Graph $MG = (V, E, W)$
\begin{algorithmic}[1]
  \For{all vertices $v$ in G}
    \State Sample length $l$ $num\_walks$ paths $P$ from $v$ \label{algrandom:find_path}
    \For{path $p=(v, x_1, ..., x_{l})$ in $P$}
      \State $score$ $=$ $ES(ET(v, x_1), 1) \cdot \prod\limits_{i=2}^{l} ES(ET(x_{i-1},x_{i}), i)$ \label{algline:rw_paths}
      \State add edge $(v, x_l)$ to $MG.E$, if it already doesn't exist
	\State  $MG.W(v, x_{l}) += score$     \label{algrandom:v_metapath_add}
    \EndFor
  \EndFor
\end{algorithmic}
\end{algorithm}
\begin{algorithm}[H]
\small
\caption{\small Metapath Graph Generation with $l/2$ Paths}\label{alg:opt_gtn}
\textbf{Input:} Graph $G$; Edge Score Function $ES$; Edge Type Function $ET$; \\
\textbf{Output:} Metapath Graph for first-half $l/2$ paths $MG_{1}$; Metapath Graph for second-half $l/2$ paths $MG_{2}$; Metapath Graph for full $l$ paths $MG$
\begin{algorithmic}[1]
  \For{all vertices $v$ in G}
    \State Enumerate length $l/2$ paths $P$ from $v$ \label{algline:o_enum}
    \For{path $p=(v, x_1, ..., x_{l/2})$ in $P$}
      \State $score_1$ $=$ $ES(ET(v, x_1), 1) \cdot \prod\limits_{i=2}^{l/2} ES(ET(x_{i-1},x_{i}), i)$ \label{algline:o_s1}
      \State $score_{2}{=}ES(ET(v, x_1), l/2 {+} 1) \cdot$ $\prod\limits_{i=2}^{l/2} ES(ET(x_{i-1},x_{i}), l/2+i)$ \label{algline:o_s2}
      \State add edge $(v,x_{l/2})$ to $MG_1.E$ and $MG_2.E$, if it already doesn't exist
      \State $MG_{1}.W(v, x_{l/2}) += score_1$ \label{algline:o_metaadd_start}
      \State $MG_{2}.W(v, x_{l/2}) += score_2$ \label{algline:o_metaadd_end}
    \EndFor
  \EndFor
  \For{metapath edge $e_1 = (a, b)$ in $MG_1$}\label{algline:o_final_start}
    \For{metapath edge $e_2 = (b, c)$ in $MG_2$}
    		\State add edge $(a,c)$ to $MG$, if it already doesn't exist
      \State $MG.W(a, c) += MG_1.W(e_1) * MG_2.W(e_2)$
    \EndFor
  \EndFor\label{algline:o_final_end}
\end{algorithmic}
\end{algorithm}



\subsection{Implementation Considerations for Pathfinding Formulation}
\label{subsec:meta-impl}

Obtaining an efficient implementation requires careful attention to
the following issues, some of which involve space-time tradeoffs.

\paragraph{Going Beyond $\mathbf{l/2}$}
Splitting the path into two $l/2$ subpaths avoids redundantly
finding $l/2$ subpaths while finding the length $l$ path.  It is possible to
use even smaller subpaths (e.g., $l/4$, $l/8$, $l/16$).
We use $l/2$ for two main reasons: (1)
breaking a path down further would require another intermediate metapath
graph object (e.g., MG$_{\{1,2,3,4\}}$) which could increase the space
overhead significantly and (2) in our applications, metapath lengths are not very large.

\paragraph{Enumeration of Paths}
The performance of metapath-finding depends critically on the method
used for path enumeration, and involves a trade-off between time and
space.  For example, paths from a given node can be enumerated by
performing a depth-first walk starting at that node. This is memory-efficient
since there are only a limited number of ``active path'' searches at any given time.
The problem with this approach is that there
may be redundancy when computing path scores.
To illustrate, say there are two paths $\{A,B,C,D,E\}$ and $\{A,F,C,D,E\}$
found by two different threads.  Both threads would find the
subpath $\{C,D,E\}$ to append to the subpaths
$\{A,B,C\}$ and $\{A,F,C\}$. It is more efficient to find this subpath
once and multiply its score with the sum of the two prefix paths.
An alternative is to perform level-by-level path enumeration. First, find all length
$k$ metapaths and their scores and store them all in memory.
Then, to get length $k {+} 1$ metapaths, take the length $k$ metapaths, extend them by one edge,
and multiply with the score for that edge.  This avoids the
redundancy problem of the depth-first extension: $\{A,B,C\}$ and
$\{A,F,C\}$ will be found at the same time with their contributions present
in the length 2 metapath edge $(A,C)$.  This approach, however, must store
all such intermediate metapaths during computation, and this space overhead
quickly becomes infeasible as the path length and the graph grows in size.


\paragraph{Memoization of Paths vs. Recomputation of Paths}

The gradient update step of GTN training requires that intermediate metapaths of
length less than $l$ are available to compute the gradients to update edge scores.
PyTorch GTN's~\cite{gtn} auto-differentiation mechanism stores the intermediate matrices for
metapaths of length 2 to length $l$ metapaths, and this becomes prohibitively expensive
for large graphs and metapath lengths.  In the
graph formulation, one can store all length $l$ paths and derive gradients from
these paths, but this is also expensive.  To avoid memory overhead, our implementation regenerates all paths in the backward gradient update pass.  This adds computational overhead
but reduces memory overheads, so it permits us to handle much larger graphs.


\paragraph{Storage Format of (Intermediate) Metapath Graphs}

Graphs are typically stored in memory in a compressed format where non-existing
edges do not take up storage.  This is in contrast to a dense adjacency matrix
in which non-existent edges are stored explicitly as zeroes.  The first approach
has the disadvantage that the number of edges each vertex has must be known
ahead of time, so dynamic allocation of this data structure cannot be done:
precomputation must be done to determine how large an intermediate metapath
graph will be if this storage format is used.  The second approach does not have this
disadvantage, but the space requirement is proportional to the square of the number of
vertices.  The PyTorch implementation~\cite{gtn} uses dense matrices for each
intermediate metapath graph, so the memory overhead is impractical for large graphs.
Our graph formulation precomputes the space required for the metapath graph and the intermediate metapath graphs. 

\subsection{Metapath Sampling}

As graphs grow in size, the number of metapaths grows 
exponentially, and as the number of paths grow, the amount of total
computation and the overall memory usage also increase.
To overcome these issues and generate the metapath graph more
efficiently, we use a random walk based approach for sampling important paths
based on the edge scores of the paths.  This approach allows us to significantly
reduce the total amount of computation and memory usage while also allowing us
to build a metapath graph that gives comparable accuracy to original GTN
formulation. Because we only find a constant number of paths per node, we can
explicitly store them as well so recomputation of paths is not needed 
during the backward step. Algorithm~\ref{alg:gtn_random} presents the pseudocode
of our random walk approach for generating metapath graphs. The appendix
contains the details of the random walk implementation. Note that sampling
metapaths cannot be easily expressed using a matrix-based API such as PyTorch~\cite{gtn}:
sampling is a fine-grained operation while matrix operations are inherently bulk
operations. 

\section{Experimental Results}
\label{sec:results}

\subsection{Setup}
We implement our graph transformer network (GTN) implementation as well as the
graph convolutional network (GCN)~\cite{GCN} subroutine in C++ code using the
KatanaGraph graph engine: we refer to it in this paper as Graph-GTN (G-GTN). We
also evaluate the random walk variant of G-GTN which we refer to as Walk-GTN
(W-GTN). W-GTN-X refers to W-GTN with X sampled metapaths for each vertex. Our
GCN implementation is based on the vertex program formulation in
DeepGalois~\cite{deepgalois} with an extension to propagate the gradients to the
GTN phase during training; this formulation has been shown to be competitive
with the state-of-the-art. We compare our GTN with the PyTorch implementation
which we refer to as P-GTN.  \cref{table:inputs} lists the inputs used in our
experiments.  We use ACM, DBLP, and IMDB~\cite{HGAN}, three inputs from the
original GTN paper~\cite{gtn}, to evaluate the runtime and accuracy. In order to
evaluate scalability for large graphs, we add synthetic labels/features/edge
types (as required) to chem2bio~\cite{chem2bio2rdf, chembioedge2vec},
reddit~\cite{GraphSAGE,pushshift}, ogbn-products~\cite{OGB}, and
twitter40~\cite{twitter40}; we do not evaluate accuracy on these graphs because
of this synthetic  metadata.
We were unable to evaluate the P-GTN with these inputs because it
failed to run them due to the memory requirements of that implementation.

\begin{figure*}[ht]
	\begin{center}
		\includegraphics[width=0.98\textwidth]{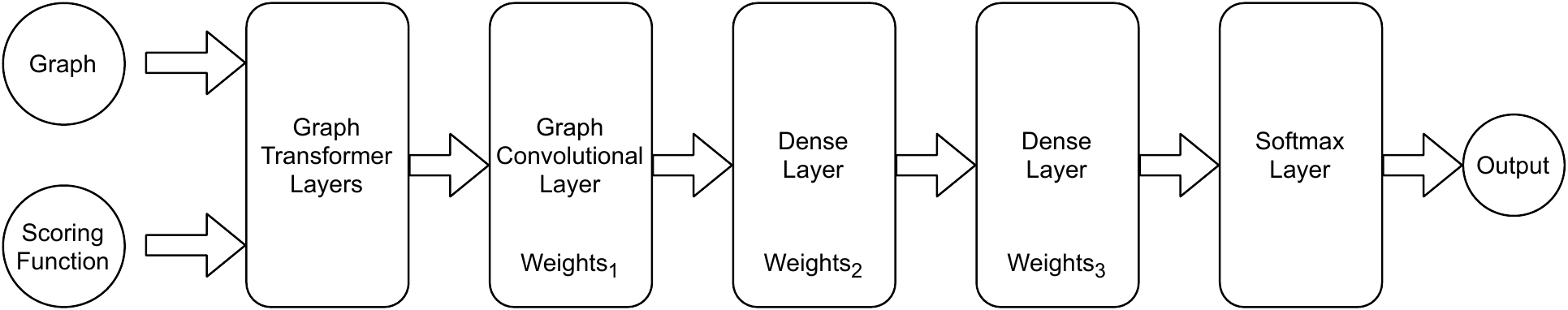}
		\caption{Illustration of the Graph Transformer Network
		architecture. The Graph Transformer Layers (one layer for each metapath
		edge) output a metapath graph for use by the GCN and dense layers.}
		\label{fig:gtn-arch}
	\end{center}
\end{figure*}

Experiments are run on a machine with 40 Intel Xeon Gold 5218R (Cascade Lake)
CPUs with 256GB DRAM. All networks are run for 300 epochs: the GTNs use a GCN
layer followed by 2 dense layers (illustrated in \cref{fig:gtn-arch}) with a
hidden feature size of 64 except for twitter40, which uses 16 due to twitter40's
higher memory cost, and the GCN uses 2 GCN layers with hidden feature size 64.
The output layer is a softmax layer in all cases. If the 300 epochs take longer than 8
hours (i.e., average epoch time of 96 seconds), we list the time as
``TO''/timeout in the results. Precomputation time (e.g., graph construction,
intermediate CSR computation, etc.) for all systems is not included in the
runtime results. The accuracy we report is
for the node classification problem: test accuracy is evaluated every 5 epochs including the
last epoch. Unless otherwise mentioned, GTNs use 3 graph transformer layers
(i.e., the metapaths being considered have up to 4 edges).  More experimental
setup details can be found in the Appendix.

\begin{table}[t]
\footnotesize
\centering
\begin{tabular}{c|ccccccc}
\toprule
& \multicolumn{7}{c}{\textbf{Dataset}} \\ \hline

                                        & {\textbf{ACM}}          & {\textbf{IMDB}}             & {\textbf{DBLP}}            & {\textbf{chem2bio}}  & {\textbf{reddit}}       & {\textbf{ogbn-products}}    & {\textbf{twitter40}}            \\ \hline
\textbf{Vertices}                       &       8994              &          12772              &           18405            & 296K                 & 232K                    & 2.4M                        & 61.6M                           \\
\textbf{Edges}                          &         25922           &          37288              &              67946         & 728K                 & 114.6M                  & 123.7M                      & 1,468M                          \\
\textbf{Features}                       &           1902          &            1256             &                  334       & -                    & -                       & -                           & -                               \\
\textbf{Classes}                        &            3            &         3                   &            3               & -                    & -                       & -                           & -                               \\
\textbf{Edge Types}                     & 4                       & 4                           &    4                       & -                    & -                       & -                           & -                               \\
\textbf{Train/Val/Test\%}    &      7/3/24             & 2/2/18                      &     4/2/16                 & -                    & -                       & -                           & -                               \\
\bottomrule
\end{tabular}
\caption{Input graphs and their properties. chem2bio, reddit, ogbn-products, and
twitter40 have synthetic data for values it does not explicitly have (e.g.,
twitter40 has no features); these values are omitted from the table since
accuracy is not evaluated. The sum of train, val, and test splits does not amount
to 100 since not all the nodes' labels are used.}
\label{table:inputs}
\end{table}

\subsection{Comparisons with matrix-based GTN}

\begin{table}[t]
\centering
\footnotesize
\begin{tabular}{l|rr|rr|rr}
\toprule
                   & \multicolumn{2}{c|}{\textbf{ACM}}                                                              & \multicolumn{2}{c|}{\textbf{IMDB}}                                                             & \multicolumn{2}{c}{\textbf{DBLP}}                                                                      \\ \cline{2-7} 
                   & \multicolumn{1}{c}{\textbf{Time (s)}} & \multicolumn{1}{c|}{\textbf{Accuracy}}                 & \multicolumn{1}{c}{\textbf{Time (s)}} & \multicolumn{1}{c|}{\textbf{Accuracy}}                 & \multicolumn{1}{c}{\textbf{Time (s)}} & \multicolumn{1}{c}{\textbf{Accuracy}} \\ \hline
\textbf{GCN}       & 0.05                                             & 0.92                                        & 0.04                                             & 0.57                                        & 0.03                                             & 0.89                                       \\
\hline
\textbf{P-GTN}     & 12.20                                            & 0.90                                        & 32.25                                            & 0.57                                        & 73.03                                            & 0.94                                       \\
\textbf{G-GTN}     & 6.88                                             & 0.90                                        & 0.26                                             & 0.57                                        & 56.68                                            & 0.95                                       \\
\textbf{W-GTN-10}  & 0.06                                             & 0.90                                        & 0.10                                             & 0.54                                        & 0.11                                             & 0.90                                       \\
\textbf{W-GTN-50}  & 0.14                                             & 0.92                                        & 0.17                                             & 0.59                                        & 0.32                                             & 0.94                                       \\
\textbf{W-GTN-100} & 0.20                                             & 0.92                                        & 0.24                                             & 0.60                                        & 0.57                                             & 0.94                                       \\
\bottomrule
\end{tabular}
\caption{Average epoch time and peak accuracy across 300 epochs with GTNs with 3 graph transformer layers (i.e., metapath with up to 4 edges).}
\label{tbl:metapath3}
\end{table}

\paragraph{Overview} \cref{tbl:metapath3} shows the average epoch time and peak
accuracy for P-GTN compared to G-GTN, W-GTN, and the basic GCN when metapaths
with 4 edges are used. G-GTN and W-GTN, the graph formulations of the GTN
problem, are significantly faster than P-GTN: on average, G-GTN is
$6.5\times$ faster, W-GTN-10 is $350\times$ faster, W-GTN-50 is
$155\times$ faster, and W-GTN-100 is $101\times$ faster than P-GTN for
this setting.  W-GTN outperforms the other GTN systems: it looks for a constant
number of metapaths per node, so for non-trivial graphs with many paths, the
reduction in the number of paths results in better runtime.

Differences in accuracy of P-GTN and G-GTN can be attributed to randomness such
as weight initialization or the differences in exact computation among the
systems (e.g., floating point inaccuracies based on different multiplication
order). The important point to note is that accuracy is similar between the two;
this should be the case because abstractly, they do the exact same thing.  On
the other hand, in principle, sampling in W-GTN may affect accuracy since the
metapath graph generated will consist of only a constant number of random
sampled walks from each vertex.  What is noteworthy is that W-GTN does not
significantly degrade accuracy even though the metapath graph is constructed
from sampled paths; in these results, degradation occurs only when 10 paths per
node are sampled, but there is no loss of accuracy when more paths are sampled.
In fact, W-GTN can achieve \emph{higher} peak accuracy in these settings. This
can occur because \emph{important} metapaths are found by sampling, so the
resulting graph is less noisy than both the original graph and the full metapath
graph.  This could also be due to the fact that sampling can avoid overfitting
to the training dataset.

\begin{figure*}[ht]
	\begin{center}
		\includegraphics[width=0.40\textwidth]{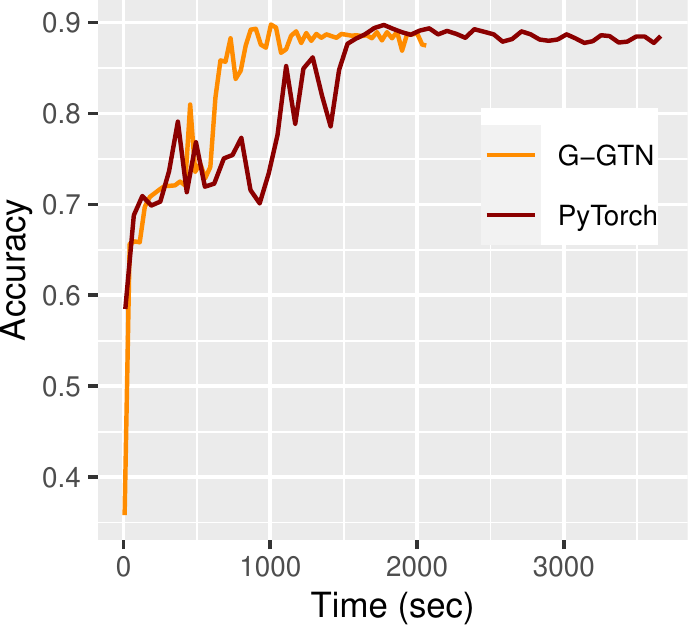}
		\includegraphics[width=0.40\textwidth]{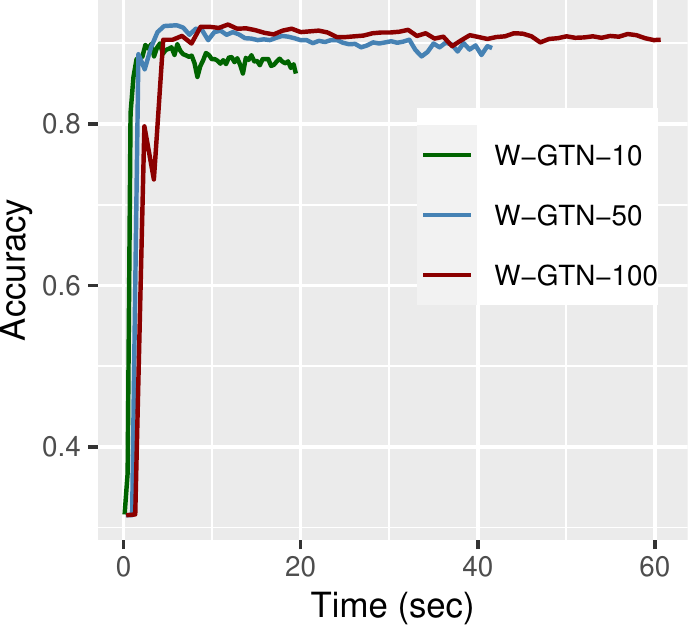}
		\caption{Time to accuracy for GTN based systems for runs with 3 metapath
		layers.}
		\label{fig:time_to_acc_acm}
	\end{center}
\end{figure*}

\begin{figure*}[ht]
	\begin{center}
		\begin{subfigure}[b]{0.32\textwidth}
		\includegraphics[width=\textwidth]{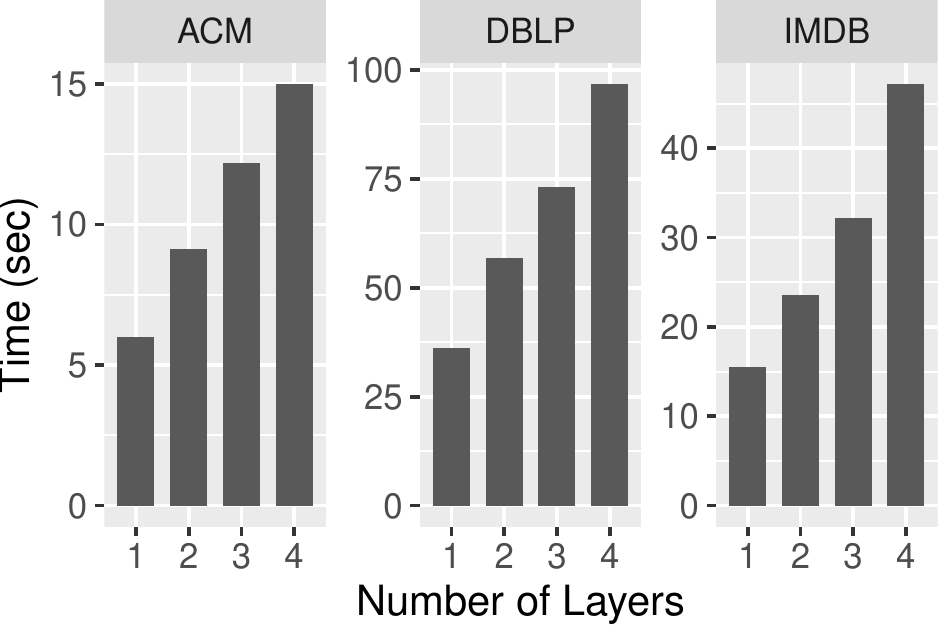}
		\caption{P-GTN}
		\end{subfigure}
		\begin{subfigure}[b]{0.32\textwidth}
		\includegraphics[width=\textwidth]{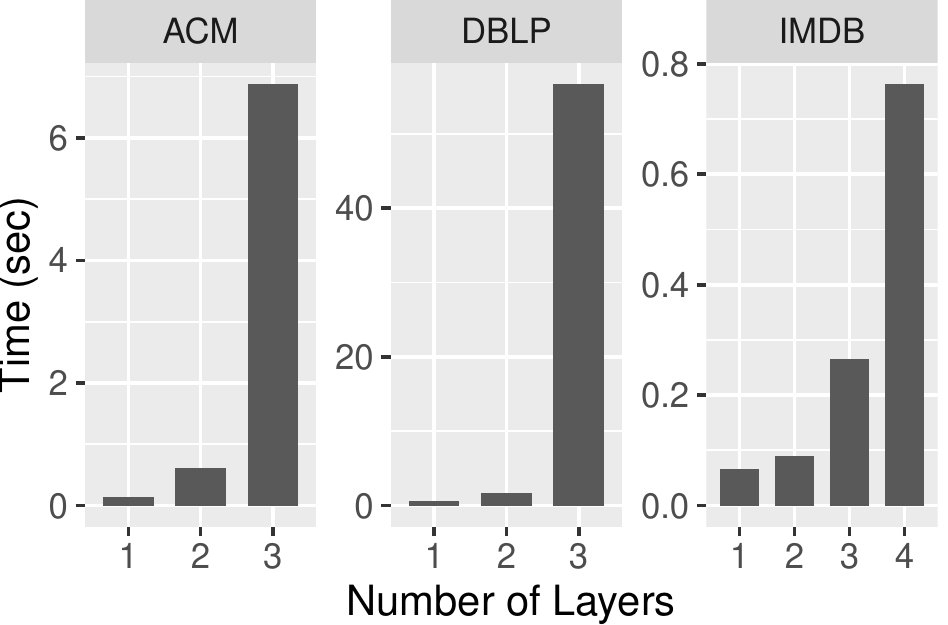}
		\caption{G-GTN}
		\end{subfigure}
		\begin{subfigure}[b]{0.32\textwidth}
		\includegraphics[width=\textwidth]{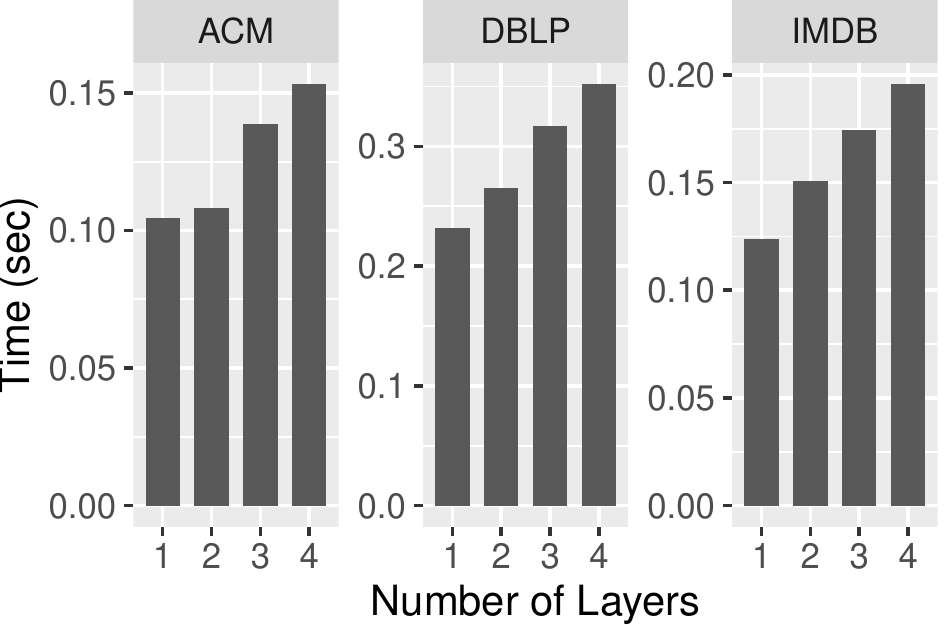}
		\caption{W-GTN-50}
		\end{subfigure}
		\caption{Average epoch time scaling of GTN systems for varying metapath
		lengths. Note that scales are different in each plot.}
		\label{fig:small_scaling}
	\end{center}
\end{figure*}

\cref{fig:time_to_acc_acm} shows the time to accuracy for the ACM graph for the
GTN-based systems. G-GTN reaches high accuracy significantly faster than P-GTN
because each epoch takes less time.  Similarly, the variations of W-GTN reach
good accuracy faster than G-GTN because of faster epoch time.

Finally, note that the GCN is the most computationally efficient of the
GNN architectures we compare with because it does not find metapaths,
and for some of these heterogeneous graphs like ACM, it obtains
comparable peak accuracy to the GTNs without any need to use the heterogeneity
of the graph. In these cases, a GTN is not required. For graphs like IMDB and DBLP where
using heterogeneous information leads to significantly better accuracy, however,
a computationally efficient GTN like G-GTN or W-GTN is vital.

\paragraph{Scaling with Metapath Length}

\cref{fig:small_scaling} shows the runtime of the GTN
systems with varying numbers of metapath lengths. P-GTN's runtime scales
linearly as the metapath length increases: increasing $l$ only adds another
dense matrix for matrix multiplication. The trade-off is that memory usage
increases significantly (by the size of the dense matrix). G-GTN is faster than
P-GTN at lower metapath lengths, but its performance declines as
metapath length increases. As explained in \cref{subsec:meta-impl}, in exchange
for memory efficiency, G-GTN must do redundant computation during path search
and does not memoize paths, resulting in an increase in compute time as the
metapath length (i.e., the number of paths) grows.  W-GTN avoids these issues
entirely by limiting the number of metapaths to a constant number and storing
these paths.  W-GTN is significantly faster than the other systems and continues
to scale as the metapath length grows because the number of generated metapaths is
constant without loss of accuracy.

\subsection{GTNs on Large Graphs}

\begin{figure*}[ht]
	\begin{center}
		\includegraphics[width=0.24\textwidth]{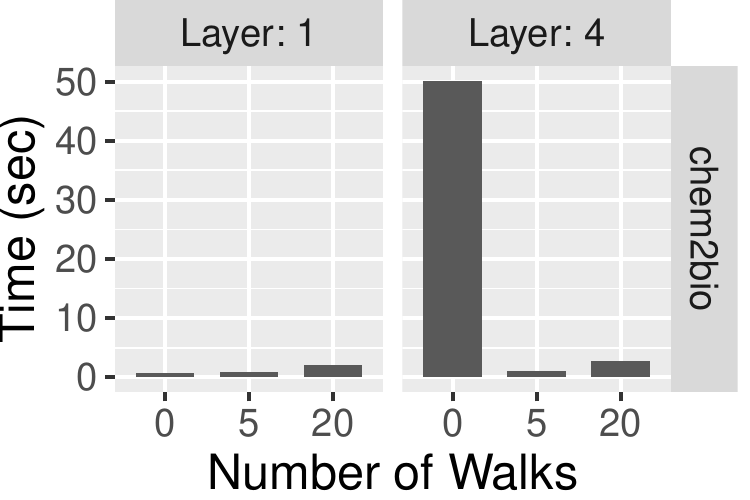}
		\includegraphics[width=0.24\textwidth]{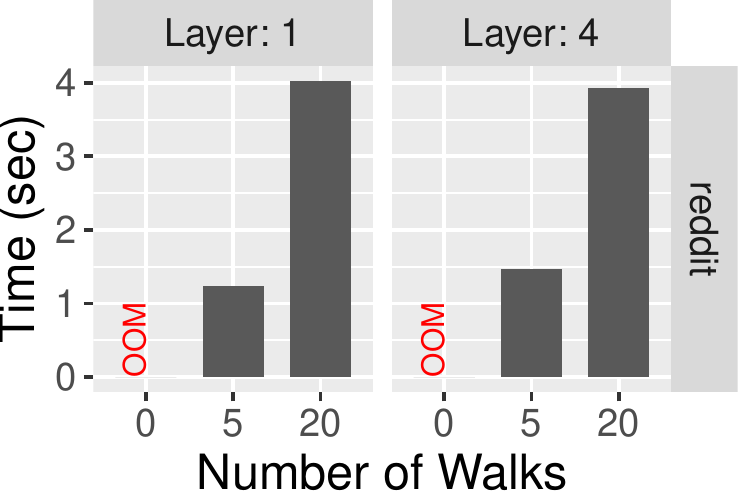}
		\includegraphics[width=0.24\textwidth]{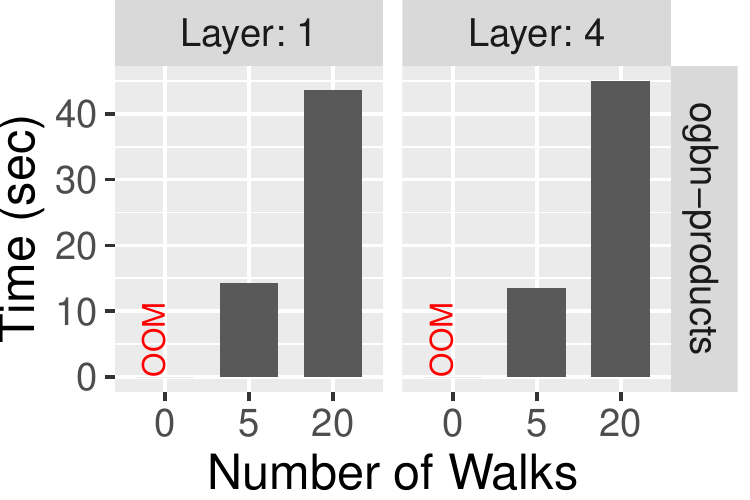}
		\includegraphics[width=0.24\textwidth]{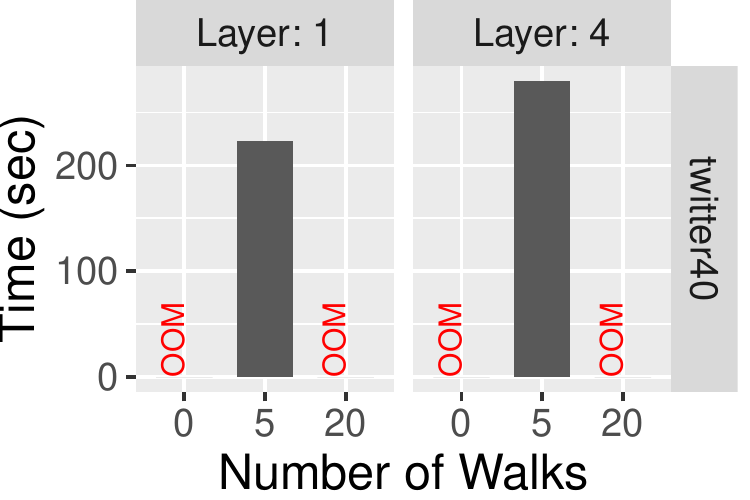}
		\caption{Scaling of W-GTN for large graphs as metapath length (i.e.,
		number of transformer layers) and number of walks are increased. 0 walks
		is G-GTN. Note that the scales are different in each plot.}
		\label{fig:large_scaling}
	\end{center}
\end{figure*}

We were unable to run P-GTN on large graphs because the storage requirement for
creating the metapath graph (which can have more edges than original graph) was
higher than the memory on our machine. G-GTN (i.e., 0 walks) was able
to process only chem2bio.  Therefore, we believe a sampling-based approach like W-GTN is crucial for doing machine learning with metapaths on large heterogeneous graphs.

\cref{fig:large_scaling} shows the average epoch time for 5 epochs on the larger graphs
for varying metapath length sizes and number of random walks. The major takeaway is
that W-GTN continues to scale as both metapath length and the number of random
walks grows. In fact, the increase in runtime as metapath length grows is not high:
the number of walks is low and constant, so an increase in length does not add
significant work to the system. Only increasing the number of walks increases
runtime for W-GTN for these graphs.


\section{Conclusion}
\label{sec:conclusion}

This paper presents (i) a new graph-based formulation of the GTN and (ii) a random
walk sampling approach for sampling metapaths to reduce both the computation and
memory usage of GTNs without compromising the accuracy of the underlying task.
The sampling based approach is up to $155\times$ faster on average than the
original matrix-based GTN implementation without accuracy degradation,
and it allows running GTNs on larger graphs that cannot be run using the PyTorch
implementation.




\section*{Societal Impact}

Heterogeneous knowledge graphs have become the preferred choice for representing
data in biomedical domain because they can assimilate data from many sources and
model the edge semantics needed for machine reasoning.  Metapaths can represent
hidden connections in these graphs, but are buried under the complex knowledge
graph, and hence can be difficult to interpret.  Graph transformer networks
(GTNs) provide a solution to this problem, but the current implementations fail
to work for any real-size data.  These implementations are also very
compute-intensive and requires a lot of energy, thus exacerbating the problem of
global warming and climate change.  This work addresses both these problems by
(1) reducing the training time for GTNs by 155$\times$ and (2) extending
the usability of GTNs to real-world biomedical datasets, thus helping identify
new connections between drug-target pairs, drug-side-effect pairs, drug-disease
pairs, disease-pathway pairs etc.

\bibliographystyle{acm}
\bibliography{bibfiles/refs2,bibfiles/dimitri,bibfiles/bc,bibfiles/blas,bibfiles/gnn,bibfiles/dnn,bibfiles/related,bibfiles/iss,bibfiles/numa,bibfiles/others,bibfiles/graphs,bibfiles/gpugraphs,bibfiles/fpgagraphs,bibfiles/partitioning,bibfiles/outofcore,bibfiles/resilience}


%
%
%
%
%

\section{Appendix}


\subsection{Additional Experimental Setup Details}

\paragraph{Changes to PyTorch GTN to Match Computation}

The high level computation changes are described here.  We are including a patch
file that has the exact changes to the source code as well. These changes were
made to make it so that P-GTN was closer to G-/W-GTN in terms of high level
computation.

\begin{enumerate}
	\item Learnable bias in PyTorch's linear layers was removed.
	\item Normalization done over out-edges, not in-edges.
	\item Metapath normalization removed.
	\item Learning rate/weight decay removed from Adam optimizer.
	\item Only 1 channel is used when running P-GTN (G-GTN does not
	have support for multiple channels, and multiple channels would
	mean more overhead in any case).
\end{enumerate}

\paragraph{Hyperparameter Details}

The layer composition of 1 GCN layer followed by 2 linear layers for the GTN was
chosen because the original GTN paper used these parameters.  Similarly, the
hidden feature size of 64 on each layer  was chosen because that is what the
original PyTorch GTN paper used.  All systems use the Adam optimizer with a
learning rate is 0.01, beta parameters 0.9 and 0.999, epsilon $10^{-8}$, and no
weight decay.

\paragraph{Data Split}
The data split used for ACM, DBLP, and IMDB were the same splits used in
the original Graph Transformer Network paper as that is where we obtained the inputs.

\paragraph{Synthetic Feature/Label/Edge Type Generation Details}

In order to run some inputs with the GTN, we generated synthetic values for things
like features, labels, and edge types for the large graphs (chem2bio, reddit,
ogbn-products, twitter40).  Note that we do not check accuracy for these graphs
because it would not make sense due to the generated values; we did this solely
in order to illustrate scaling of our GTN implementation on larger graphs and to
show that GTNs are very expensive to run in terms of memory and compute time
without sampling or some other method of reducing the number of metapaths. The
actual synthetic values that we generated \emph{should not be interpreted to have any
kind of meaning behind them other than filler values in order to make the GTN
run}. We are including this information here for the sake of completeness.

Synthetic features were generated for chem2bio and twitter40. Every vertex's feature
was a vector of length 50 all with the number 2.

Synthetic edge types were generated for all large graphs but chem2bio: for each edge,
we randomly choose 1 of 4 different edge types.

Synthetic labels were generated for chem2bio and twitter40. For each vertex, we choose 1 of
3 labels.

Finally, we also created a train/val/test splits for chem2bio and twitter40: 20\%
were train nodes, 10\% were validation nodes, and 70\% were testing nodes.

\subsection{Proof of $l/2$ Path Split Correctness}

\begin{theorem}
Let  $G$ be a graph and let $MG$ be its metapath graph for some length $l$.
Also, let $MG_1$ and $MG_2$ be the metapath graphs of $G$ for first-half $l/2$ and second-half $l/2$ paths respectively.
$MG(u,v)$ is the weight of the edge $(u, v)$ in the metapath graph $MG$.
Then, $MG(u,v) = \sum\limits_{k\in V} MG_1(u,k) \cdot MG_2(k,v)$ for every pair of vertices $u, v \in V$.
\end{theorem}

\begin{proof}
Consider a pair of vertices $u$, $v$ and let $\mathcal{P}_l(u,v)$ represents the set of metapaths between $u$ and
$v$ for length $l$.
Then,

\begin{align*}
MG(u,v) &= \sum\limits_{p\in\mathcal{P}_l(u,v)} score(p) \\
&= \sum\limits_{p\in\mathcal{P}_l(u,v)} score_1(u,x)\cdot score_2(x,v) \\
&\text{\hspace{1pt} (where $x$ is the node at position $l/2$ in $p$)} \\
&= \sum\limits_{x \in V} (\sum\limits_{p_1 \in \mathcal{P}_{\lceil l/2 \rceil}(u,x)} score_1(u,x)) \cdot (\sum\limits_{p_2 \in \mathcal{P}_{\lfloor l/2 \rfloor}(x, v)} score_2(x,v)) \\
&= \sum\limits_{x \in V} MG_1(u,x) \cdot MG_2(x,v)
\end{align*}
\end{proof}

\subsection{Random Walk Algorithm}

\begin{algorithm}[H]
\footnotesize
\caption{Random Walk Algorithm}\label{alg:random-walk}
\textbf{Input:} Graph $G$; vertex $v$; Number of walks $num\_walks$; Walk-length $walk\_length$\\
\textbf{Output:} Set of paths $\mathcal{P}$
\begin{algorithmic}[1]
 \State $\mathcal{P} \leftarrow \phi$
  \For{$i \leftarrow 1$ to $num\_walks$}	\label{alg:for}
  	\State $p \leftarrow \{v\}$
  	\For{$j \leftarrow 1$ to $walk\_length$}	\label{alg:for1}
  		\State $last \leftarrow p[j-1]$
  		\State Randomly sample $u$ from $\mathcal{N} \cup \{last\}$ using acceptance-rejection sampling
  		\State $p \leftarrow p \cdot (last, u)$
  	\EndFor	\label{alg:endFor1}
  	\State $\mathcal{P} \leftarrow \mathcal{P} \cup p$
    \State \textbf{return} $\mathcal{P}$
  \EndFor	\label{alg:endFor}
\end{algorithmic}
\end{algorithm}

We give a brief overview of our random walk implementation.  The input is a
graph $G$ along with a vertex $v$ from which we need to sample $num\_walks$
paths of length $walk\_length$.  Each iteration of the for loop
(Lines~\ref{alg:for}-\ref{alg:endFor}) samples one path, starting from $v$ of
length $walk\_length$.  In the inner for loop in
lines~\ref{alg:for1}-\ref{alg:endFor1}, a vertex is sampled randomly using an
acceptance-rejection sampling technique, which takes into account the weights of
outgoing edges. Edges with higher weights are more likely to get picked, meaning
more important metapaths (at that point in training/inferencing) are more likely
to be sampled.

\end{document}